\documentclass[onecolumn]{IEEEtran}
\usepackage{amssymb}
\usepackage{amsfonts,amsmath}

\usepackage{algorithmic}
\usepackage{algorithm}
\usepackage{hyperref}

\newtheorem{thm}{Theorem} 
\newtheorem{lem}[thm]{Lemma}

\newtheorem{prop}[thm]{Proposition}

\newtheorem{dftemp}[thm]{Definition}
\newtheorem{extemp}[thm]{Example}
\newtheorem{rmktemp}[thm]{Remark}
\newtheorem{convtemp}[thm]{Convention}
\newenvironment{defn}{\begin{dftemp}\normalfont}{\end{dftemp}}

\newenvironment{rem}{\begin{rmktemp}\normalfont}{\end{rmktemp}}

\begin{document}
 
\title{ A Polynomial Time Approximation Scheme for a Single Machine Scheduling Problem Using a Hybrid Evolutionary Algorithm }

\author{Boris Mitavskiy and Jun He~\thanks{Boris Mitavskiy and Jun He are with Department of Computer Science, Aberystwyth University, Aberystwyth, SY23 3DB, U.K.  }}

\maketitle

\begin{abstract}
Nowadays hybrid evolutionary algorithms, i.e, heuristic search algorithms combining several mutation operators some of which are meant to implement stochastically a well known technique designed for the specific problem in question while some others playing the role of random search, have become rather popular for tackling various NP-hard optimization problems. While empirical studies demonstrate that hybrid evolutionary algorithms are frequently successful at finding solutions having fitness sufficiently close to the optimal, many fewer articles address the computational complexity in a mathematically rigorous fashion. This paper is devoted to a mathematically motivated design and analysis of a parameterized family of evolutionary algorithms which provides a polynomial time approximation scheme for one of the well-known NP-hard combinatorial optimization problems, namely the ``single machine scheduling problem without precedence constraints". The authors hope that the techniques and ideas developed in this article may be applied in many other situations.  
\end{abstract}


\section{Introduction}
Scheduling problems appear naturally in a variety of applications whenever one needs to make decisions about the order in which tasks (such as parallel subroutines in computer programs, transportation of passengers or goods, arranging examination timetables in universities etc.) are to be executed economically subject to limited resources (such as limited number of processors, vehicles, employees, etc.) and other constraints (such as the timing when various tasks become available or dependence of a certain task upon the outcome of another task etc.) As it is usually the case in combinatorial optimization, most scheduling problems are NP-hard (see, for instance \cite{LeslieHall} and \cite{LenstraScheduleComplexity}) and, therefore, various heuristic search algorithms are frequently exploited to tackle them. Among these are the so-called hybrid evolutionary algorithms. Such techniques combine random mutation or recombination operators with other algorithms (such as various greedy approaches, for instance). A hybrid evolutionary algorithm for Graph Coloring developed in \cite{galinier1999hybrid} embeds local search into the framework of evolutionary algorithms. This hybrid evolutionary algorithm combines a new class of highly specialized crossover operators and a well-known tabu search algorithm. Experiments with such a hybrid algorithm have been carried out on large DIMACS Challenge benchmark graphs and the results have proven very competitive with and even better than those of state-of-the-art algorithms. A Hybrid evolutionary algorithm for Job Scheduling presented in \cite{xhafa2007hybrid} combines memetic algorithms and several local search algorithms. The memetic algorithm is used as the principal heuristic that guides the search and can use any of the 16 local search algorithms during the search process. The local search algorithms used in combination with the MA have been obtained by fixing either the type of the neighborhood or the type of the move; these include swap/move based search, Hill Climbing, Variable Neighborhood Search, and Tabu Search. A popular approach to design hybrid evolutionary algorithms is to combine local search with evolutionary algorithms.

Application-specific, parameterized local search algorithms (PLSAs), in which optimization accuracy can be traded off with run time, arise naturally in many optimization contexts. In \cite{bambha2004systematic}, a novel approach, called simulated heating, for systematically integrating parameterized local search into evolutionary algorithms (EAs) has been introduced. Using the framework of simulated heating, both static and dynamic strategies for systematically managing the tradeoff between PLSA accuracy and optimization effort have been investigated. The goal was to achieve maximum solution quality within a fixed optimization time budget. It has been shown that the simulated heating technique better utilizes the given optimization time resources than standard hybrid methods that employ fixed parameters, and that the technique is less sensitive to these parameter settings. This framework has been applied to three different optimization problems, and the results have been compared to the standard hybrid methods, showing quantitatively that careful management of this tradeoff is necessary to achieve the full potential of an EA/PLSA combination.

Despite their frequent success the choice of a particular algorithm to optimize a certain problem is mainly driven by the intuition or personal bias of its designer towards a certain technique rather than by solid mathematical reasoning so that the reasons behind the high performance usually remain unclear (see \cite{preux1999towards} as well as \cite{Hart2005}).

In the current paper we design a parameterized family of hybrid $2+2$ EAs based on rigorous mathematical approach motivated by the theory presented in \cite{LeslieHall}, that provide a polynomial time approximation scheme for a ``single machine scheduling problem". This type of scheduling problem will be described in detail in section~\ref{technicalPartSect}, nonetheless it is important to know that the problem has been shown to be strongly NP-hard in \cite{LenstraScheduleComplexity} via a polynomial time reduction from the classical 3-PARTITION problem. Our hybrid $2+2$-EA combines a local mutation operator which is based on a greedy heuristic called the Jackson rule together with a given approximation ratio parameter $\epsilon>0$ and a global mutation operator applied with relatively small probability with the aim to explore all the possible local optima. At the end of section~\ref{technicalPartSect} we prove (see theorem~\ref{finalGoalThm}) that given any specified approximation parameter $\epsilon > 0$, our $2+2$-EA finds a schedule within the $(1+\epsilon)$ factor from an optimal schedule after $n^{\frac{1}{\epsilon} + 7}+n^5$ time steps with overwhelmingly high probability: at least $1 - \left(\frac{1}{k \cdot n}\right)^n$ for all sufficiently large $n$, where $k$ is a constant while $n$ is the total number of jobs in the schedule. As a matter of fact, thanks to parts~\ref{exponentiationOfe} and \ref{exponentiationOfeDualEq} of theorem~\ref{propOfe}, the constant $k$ can be made arbitrarily close to $1$ as long as it is strictly smaller than $1$. As a compromise, making $k$ closer to $1$ increases the lower bound on the total number of jobs $n$ beyond which the complexity bound holds.
\section{Asymptotic Notation and Elementary Mathematical Tools Involved in the Design and Complexity Analysis}\label{toolsSect}
The following notation is frequently used in algorithm complexity theory.
\begin{defn}\label{asymptoticNotationDefn}
Given sequences of real numbers $f(n)$ and $g(n)$ (i.e. functions $f \text{ and } g: \mathbb{N} \rightarrow \mathbb{R}$) we say that $f = O(g)$ if $\exists \, k>0$ and $N > 0$ such that $\forall \, n > N$ we have $f(n) \leq k \cdot g(n)$. Dually, we say that $f = \Omega(g)$ if $\exists \, k>0$ and $N > 0$ such that $\forall \, n > N$ we have $f(n) \geq k \cdot g(n)$.
\end{defn}
Recall from elementary analysis that the irrational number $\lim_{n \rightarrow \infty}\left(1 + \frac{1}{n}\right)^n \approx 2.71$ up to two significant digits exists and is commonly denoted by the letter $e$ (see, for instance, \cite{RealAnalysis}). The elementary facts summarized in theorem~\ref{propOfe} below will be frequently exploited throughout the next section.
\begin{thm}\label{propOfe}
The following facts are true:
\begin{equation}\label{exponentiationOfe}
\forall \, x \in \mathbb{R} \text{ we have } \lim_{n \rightarrow \infty}\left(1 + \frac{x}{n}\right)^n = e^x.
\end{equation}
In particular, when $x=-1$ we have
\begin{equation}\label{exponentiationOfePart}
\lim_{n \rightarrow \infty}\left(1 - \frac{1}{n}\right)^n = e^{-1}.
\end{equation}
Let $\beta > \alpha > 0$ and consider the sequence
\begin{equation}\label{exponentiationOfeEqDefn}
s_n = \left(1 - \frac{1}{n^{\alpha}}\right)^{n^{\beta}} = \left(\left(1 - \frac{1}{n^{\alpha}}\right)^{n^{\alpha}}\right)^{n^{\beta - \alpha}}.
\end{equation}
Then $\forall \, k < 1$ $\exists \, N \in \mathbb{N}$ such that $\forall \, n > N$ we have
\begin{equation}\label{exponentiationOfeEq}
s_n < e^{-k \cdot n^{\beta - \alpha}}.
\end{equation}
Dually, assume that $\alpha > \beta > 0$. Then $\forall \, k < 1$ $\exists \, N \in \mathbb{N}$ such that $\forall \, n > N$ we have
\begin{equation}\label{exponentiationOfeDualEq}
s_n > e^{-k \cdot n^{\beta - \alpha}}.
\end{equation}
\begin{equation}\label{basicInequalAboute}
\forall \, x > 0 \text{ the inequality } e^{-x} > 1 - x \text{ holds.}
\end{equation}
\end{thm}
\begin{IEEEproof}
\ref{exponentiationOfe} and \ref{exponentiationOfePart} can be found in any introductory real analysis textbook: see, for instance, \cite{RealAnalysis}. In summary, the arguments are based on binomial expansion and power series. To see inequalities~\ref{exponentiationOfeEq} and \ref{exponentiationOfeDualEq}, observe that the function $\varphi(x) = \left(1 - \frac{1}{x}\right)^x$ is strictly decreasing on the interval $(0, \, \infty)$ since $\varphi '(x) < 0$. Now let $k<1$ be given. Monotonicity together with the fact in~\ref{exponentiationOfePart} tell us that $\lim_{x \rightarrow \infty}\left(1 - \frac{1}{x}\right)^x = e^{-1}$ so that $\exists \, N \in \mathbb{N}$ large enough such that $\forall \, n > N$ we have $$e^{-k} > \left(\left(1 - \frac{1}{n^{\alpha}}\right)^{n^{\alpha}}\right) > e^{-1}$$ so that the inequalities \ref{exponentiationOfeEq} and \ref{exponentiationOfeDualEq} follow at once. The inequality~\ref{basicInequalAboute} is very commonly used in practice and has a very short proof: consider the function $f(x) = e^{-x} - (1 - x)$. Then $f(1) = 0$ and $\forall \, x>0$ we have $f'(x)>0$ so that $f$ is strictly increasing on the interval $(0, \, \infty)$. In particular, $\forall \, x \in (0, \, \infty)$ we have $f(x) > f(1) = 0$ which is equivalent to the desired inequality in \ref{basicInequalAboute}.
\end{IEEEproof}
Apart from the elementary notions summarized above, basic probability theory and basic Markov chain theory are heavily exploited in the upcoming section.
There is a plenty of literature covering both subjects in much detail: see, for instance, \cite{GrimmetStirzaker}.

\section{A Parameterized Family of Hybrid $2+2 EAs$ that Provides a Polynomial Time Approximation Scheme for Single-Machine Scheduling Problem}\label{technicalPartSect}
The ideas in this section are closely related to these presented in~\cite{LeslieHall}. Suppose we have a sequence of $n$ jobs $\{J_i\}_{i=1}^n$ where each job $J_i$ must be processed without interruption for a time $p_i > 0$ on the same machine $M$. A job $J_i$ is released at time $r_i \geq 0$ associated with it and it becomes available for processing only at the time $r_i$ and any time after as long as the machine $M$ is not occupied at the time being. As soon as a job $J_i$ has finished processing it is sent for delivery immediately. A specific job $J_i$ has its own delivery time $q_i$. Here we assume that there is no restriction on the total number of jobs being delivered simultaneously. Our objective is to find a reordering of the sequence $\{J_i\}_{i=1}^n$ of jobs which minimizes the minimal time when all of the jobs have just been delivered, referred to as the maximal lateness of the schedule.\footnote{It has been explicitly shown in~\cite{LeslieHall} that the setting above is equivalent to the model with due dates in place of the delivery times via a simple linear change of variables, yet the model with delivery times is a lot better suitable for the algorithm design and analysis.} In summary, given a set of $n$ ordered triplets $\{J_i\}_{i=1}^n$ with each $J_i = (r_i, \, p_i, \, q_i)$ that stand for \emph{release time}, \emph{processing time} and \emph{delivery time} of the job $J_i$ respectively, our aim is the following: given $\epsilon > 0$, produce a $\mu+\mu$ hybrid EA which produces a schedule (i.e. a feasible\footnote{In the special case of single machine scheduling with no preprocessing constraints studied in the current article, as we will see, every permutation of the sequence $\{J_i\}_{i=1}^n$ of jobs produces a feasible schedule.} reordering/permutation $\pi_t$ of the sequence $\{J_i\}_{i=1}^n$ in time $t$ depending polynomially on $n$ and $\frac{1}{\epsilon}$ which, in turn, determines the new sequence $(J_{\pi_t(i)})_{i=1}^n$) the maximal lateness of which, call it $L_t$, is such that
\begin{equation}\label{approxRatioGoalEq}
L_t \leq (1 + \epsilon) \cdot L^*
\end{equation}
where $L^*$ denotes the optimal (i.e. the minimal possibly achievable) maximal lateness. We will say that such a schedule $L_t$ is $\epsilon$-\emph{optimal}. Let $P = \sum_{i=1}^n p_i$ and $\delta = \epsilon \cdot P \leq \epsilon \cdot L^*$ so that we can write $(1 + \epsilon) \cdot L^* \geq L^* + \delta$. Thereby to deduce inequality~\ref{approxRatioGoalEq} it is sufficient to establish the inequality~\ref{approxRatioGoalAlternEq} below which is expressed purely in terms of $\delta$ and $n$:
\begin{equation}\label{approxRatioGoalAlternEq}
L_t \leq L^* +  \delta.
\end{equation}
This form is more convenient for the algorithm description and analysis. We now proceed to describe the search space in detail.

The search space consists of all possible permutations of the sequence $\{J_i\}_{i=1}^n$ of jobs.
We will denote the search space with
\begin{equation}\label{searchSpDefnEq}
\Omega = \{\pi \, | \, \pi \text{ is a permutation on }\{1, \, 2, \ldots, n\}\}.
\end{equation}
where $\pi \in \Omega$ represents the schedule $(J_{\pi(i)})_{i=1}^n$.

For theoretical convenience we introduce an extra job $J_0 = (0, \, 0, \, 0)$ into the schedule that will make no practical difference at all.
\begin{defn}\label{theoretConvenienceDef}
Extend the initial sequence $\{J_i\}_{i=1}^n$ of jobs to include the job $J_0$ having the property that $r_0 = p_0 = q_0 = 0$ (i.e. the new initial schedule is now $\{J_i\}_{i=0}^n$ where $J_0 = (0, \, 0, \, 0)$). We extend every permutation $\pi$ on $\{i \, | \, 1 \leq i \leq n\}$ to the permutation $\pi$ on $\{i \, | \, 0 \leq i \leq n\}$ trivially by letting $\pi(0) = 0$ (in fact, this is obviously the unique extension).
\end{defn}
Our hybrid $2+2$ EA will ensure that we deliver the $\epsilon$-optimal schedule within time polynomial in $n$ depending on $\delta$. In order to construct and to analyze (as well as to motivate the construction of) the $2+2$ EA it is convenient to introduce the following recursive definitions:
\begin{defn}\label{startingTimeDef}
We will write $s_i(\pi)$ to denote the \emph{starting time} of the job $J_{\pi(i)}$ within the schedule $\pi$. Let $s_0(\pi) = 0$. For $i \geq 1$ let $s_i(\pi) = \max \{s_{i-1}(\pi) + p_{\pi(i-1)}, \, r_{\pi(i)}\}$. Let $L_{\pi} = \max_{1 \leq i \leq n} s_i(\pi) + p_{\pi(i)} + q_{\pi(i)}$ denote the \textit{maximal lateness} (i.e. the earliest time when all the jobs have just been delivered) of the schedule $\pi$.

When the schedule $\pi$ is clear from the context we will write $s_i$ in place of $s_i(\pi)$.
\end{defn}
Definition~\ref{startingTimeDef} makes perfect sense since the jobs can be started only after they have been released and there is no point to wait for a job to start unless the machine is occupied. It is worth pointing out the following.
\begin{rem}\label{RecursiveCasesRem}
$\forall \, i$ with $1 \leq i \leq k$ there are three possible mutually exclusive cases.

\textbf{Case 1:} $s_{i-1} + p_{\pi(i-1)} < r_{\pi(i)}$. In this case $$s_i = r_{\pi(i)} > s_{i-1} + p_{\pi(i-1)}.$$

\textbf{Case 2:} $s_{i-1} + p_{\pi(i-1)} > r_{\pi(i)}$. In this case $$s_i = s_{i-1} + p_{\pi(i-1)} > r_{\pi(i)}.$$

\textbf{Case 3:} $s_{i-1} + p_{\pi(i-1)} = r_{\pi(i)}$. In this case $$s_i = s_{i-1} + p_{\pi(i-1)} = r_{\pi(i)}.$$
\end{rem}
\begin{defn}\label{PartlyJacksonDynamic}
Given $\delta > 0$, let $A_{\delta} = \{i \, | \, p_i < \delta\}$ and $B_{\delta} = \{i \, | \, p_i \geq \delta\}$ denote the indices of these jobs in the initial sequence that have processing times shorter then $\delta$ and at least as long as $\delta$ respectively. Given a subset of indices $I = \{i_1, \, i_2, \ldots, i_b\} \subseteq \{1, \, 2, \, 3, \ldots, n\}$ with $|I| = |B_{\delta}| = b$ and a bijection $\phi: I \rightarrow B_{\delta}$ we will say that a schedule (a permutation determining the schedule) $\pi$ is $(k, \, \delta, \, \phi)$-Jackson if $\forall \, j$ with $1 \leq j \leq k$ the following are true:

$\;$

If $j = i_q \in I$ then $\pi(j) = \phi(i_q)$.

If $j \notin I$ then exactly one of the three cases above takes place with $j$ playing the role of $i$.

$ \; $ If case 1 takes place then we require that $\forall \, l>i$ we have $r_{\pi(l)} \geq r_{\pi(j)}$ and, in case if $r_{\pi(j)} = r_{\pi(l)}$ we must have $q_{\pi(l)} \leq q_{\pi(j)}$.

$ \; $ If cases 2 or 3 take place then we require that $\forall \, l>j$ we have either $r_{\pi(l)} > s_{j}(\pi)$ or ($r_{\pi(l)} \leq s_{j}(\pi)$ and $q_{\pi(l)} \leq q_{\pi(j)}$).

$\,$

An $(n, \, \delta, \, \phi)$-Jackson schedule will be called simply a $(\delta, \, \phi)$-Jackson schedule.
\end{defn}
\begin{rem}\label{partitionIntoPhiJacksonRem}
Notice that given any schedule $\pi$ $\exists !$ $\phi$ such that $\pi$ is a $(k, \, \delta, \, \phi)$-Jackson schedule. Indeed, $\phi$ is obtained by identifying the locations of the jobs in $B_{\delta}$ within the schedule $\pi$. Setting $k=0$ will always do, but one can agree to select the maximal $k$ such that $\pi$ is a $(k, \, \delta, \, \phi)$-Jackson schedule and such $k$ is evidently unique. From now on whenever we refer to $(k, \, \delta, \, \phi)$-Jackson schedule we will have in mind that $k$ is maximal with respect to this property.
\end{rem}
In view of remark~\ref{partitionIntoPhiJacksonRem} we have an important equivalence relation on the search space of all the schedules $\Omega$ (see \ref{searchSpDefnEq}).
\begin{defn}\label{JacksonPartitionDefn}
We say that $\pi \sim \sigma$ if both, $\pi$ and $\sigma$, are $(k_1, \, \delta, \, \phi)$-Jackson and $(k_2, \, \delta, \, \phi)$-Jackson respectively with the common unique $\phi$. Let $$\mathcal{X}(\phi) = \{\pi \, | \, \pi \text{ is }(k, \, \delta, \, \phi) \text{-Jackson where } 0 \leq k \leq n\}$$ denote the equivalence class of schedules corresponding to the given $\phi$.
\end{defn}
A crucial fact behind the design and success of our algorithm is the following clever and elegant theorem which is proved without being explicitly stated in section 1.2.3, chapter 1 of \cite{LeslieHall}.
\begin{thm}\label{mainThmForDesign}
Given any initial sequence $\{J_i\}_{i=1}^n$ of jobs with each $J_i = (r_i, \, p_i, \, q_i)$, $\forall \, \delta>0$ $\exists$ an indexing set $I \subseteq \{1, \, 2, \ldots n\}$ with $I = |B_{\delta}|$ and a bijection $\phi: I \rightarrow B_{\delta}$ such that any $(\delta, \, \phi)$-Jackson schedule is $\epsilon$-optimal.
\end{thm}
Our $2+2$-EA always keeps the fittest individual as the second individual in the population. This individual is never mutated, while the first individual is mutated at every time step. The mutant always replaces the first individual and if the mutant is fitter than the second individual it replaces the second individual as well. There will be two mutation operators involved in the algorithm: one local, driving towards achieving $(\delta, \, \phi)$-Jackson schedule, and another global: searching for a specific $\phi$ which is guaranteed to exist in theorem~\ref{mainThmForDesign}. The local mutation operator is applied with relatively large probability while the global mutation with small probability so as to ensure that the algorithm has time to find a $(\delta, \, \phi)$-Jackson schedule before it explores a different choice of the bijection $\phi$. Notice that there are $\prod_{i=0}^{|B_{\delta}|-1}(n-i) < n^{|B_{\delta}|}$ of such bijections while $$|B_{\delta}| \cdot \delta \leq \sum_{i \in B_{\delta}}p_i \leq \sum_{i=1}^n p_i = P$$ so that $|B_{\delta}| \leq \frac{P}{\delta} = \frac{1}{\epsilon}$ (see inequalities~\ref{approxRatioGoalEq} and \ref{approxRatioGoalAlternEq} and discussions preceding these inequalities) and this immediately implies the following important lemma which appears implicitly in section 1.2.3 of \cite{LeslieHall}.
\begin{lem}\label{totalNumbOfPositionsLem}
The total number of ways to position the jobs from the set $B_{\delta}$ within a schedule, (in other words, the total number of ways to select a bijection $\phi$) is bounded above by $n^{\frac{1}{\epsilon}}$.
\end{lem}
Lemma~\ref{totalNumbOfPositionsLem} tells us that the ``quotient search space" on which the global mutations act is polynomial in size and this fact will be important for the design and runtime complexity analysis of our $2+2$-EA. Another important fact contributing to the design and complexity analysis of our algorithm is the following.
\begin{prop}\label{distinctDeliveryTimesProp}
Suppose the initial sequence of jobs $\{J_i\}_{i=1}^n$ has the property that all the jobs have pairwise distinct delivery times: $\forall \, i$ and $j$ with $1 \leq i < j \leq n$ we have $q_i \neq q_j$. Choose any $k \geq 1$ and a bijection $\phi: I \rightarrow B_{\delta}$. Let $\pi$ and $\sigma$ denote any two $(k, \, \delta, \, \phi)$-Jackson schedules. Then $\forall \, i$ with $1 \leq i \leq k$ we have $s_i(\pi) = s_i(\sigma)$ and $\pi(i) = \sigma(i)$. In particular, a $(\delta, \, \phi)$-Jackson schedule is unique.
\end{prop}
\begin{IEEEproof}
We argue by the least natural number principle. Suppose proposition~\ref{distinctDeliveryTimesProp} is not true. In this case $\exists$ a minimal $j \in \{i \, | \, 1 \leq i \leq k\}$ such that either one of the assertion of proposition~\ref{distinctDeliveryTimesProp} fails. Definition~\ref{theoretConvenienceDef} tells us that $\pi(0) = \sigma(0) = 0$ so that $j > 0$ trivially. Observe also that according to definition~\ref{PartlyJacksonDynamic} $j \notin I$. Consider now the jobs $J_{\pi(j)}$ and $J_{\sigma(j)}$ where $j \notin I \cup \{0\}$. By minimality of $j$ $\forall \, i$ with $0 \leq i \leq j$ we have $\pi(i) = \sigma(i)$, $r_{\pi(i)} = r_{\sigma(i)}$, $p_{\pi(i)} = p_{\sigma(i)}$ and $s_i(\pi) = s_i(\sigma)$. First we observe that $r_{\pi(j)} \neq r_{\sigma(j)}$ (otherwise by definition~\ref{PartlyJacksonDynamic} we must have $q_{\pi(j)} = q_{\sigma(j)}$ contrary to the assumption). Next we note that one of the three cases described in remark~\ref{RecursiveCasesRem} must take place for $\pi$.

\textbf{case 1:} In this case, notice that the subcase $r_{\pi(j)} < r_{\sigma(j)}$ is impossible by definition~\ref{PartlyJacksonDynamic}. The subcase $r_{\sigma(j)} < r_{\pi(j)}$ is dual to the previous subcase when exchanging the roles of $\pi$ and $\sigma$.

\textbf{cases 2 or 3:} Without loss of generality assume that case 2 or case 3 takes place for the schedule $\sigma$ as well (otherwise apply the argument in case 1 swapping the roles of $\pi$ and $\sigma$). Now that $r_{\pi(j)} \leq s_j$ and $r_{\sigma(j)} \leq s_j$, by definition~\ref{PartlyJacksonDynamic} we must have $q_{\pi(j)} = q_{\sigma(j)}$ contrary to the assumption once again.

Thus the minimal $j \in \{i \, | \, 1 \leq i \leq k\}$ such that at least one of the assertions in proposition~\ref{distinctDeliveryTimesProp} fails has the property that $\pi(j) = \sigma(j)$. The minimality of $j$ together with definition~\ref{startingTimeDef} immediately entail the remaining property in proposition~\ref{distinctDeliveryTimesProp} that $s_i(\pi) = s_i(\sigma)$ thereby contradicting the defining property of $j$ and the desired conclusion now follows by the least natural number principle.
\end{IEEEproof}
In view of proposition~\ref{distinctDeliveryTimesProp} it is convenient to introduce the following definition:
\begin{defn}\label{genericScheduleDefn}
We say that an instance of a scheduling problem $\{J_i\}_{i=1}^n$ is \emph{generic} if $\forall \, i$ and $j$ with $1 \leq i < j \leq n$ we have $q_i \neq q_j$.
\end{defn}
While the condition of being ``generic" ensures the conclusion of proposition~\ref{distinctDeliveryTimesProp}, it is hardly a restriction for two reasons. The first reason is that tiny perturbations of delivery times will cause a very small change in the maximum lateness of the schedule and that this change can be made arbitrarily small depending on the size of the perturbation. The second reason is that one can simply introduce additional total orders on the equivalence classes of jobs having equal delivery times. Now if there is an ambiguity whether to schedule a job $J_i$ or $J_j$ in a specific position, it must be the case that $q_i = q_j$ (as we have seen from the proof of proposition~\ref{distinctDeliveryTimesProp}) so that the additional total order on the equivalence class $[J_i] = [J_j]$ will determine the job to be scheduled uniquely. Apparently any permutation of generic schedule remains generic. In view of the comments above we will assume from now on that the instance of our scheduling problem is a generic schedule.

We introduce two mutation operators, one ``local" and one ``global".

\textbf{Local Mutation Operator $\mu^{\text{local}}_{\delta}$: } Given a schedule $\pi$ select indices $i$ and $j$ where $1 \leq \pi(i) < \pi(j) \leq n$ uniformly at random.

$\;$

\textbf{IF} $p_{\pi(i)} < \delta$ and $p_{\pi(j)} < \delta$ (i.e. if $\pi(j)$ and $\pi(i) \in A_{\delta}$ which is the same as saying that $\pi(i) \notin I$ and $\pi(j) \notin I$) \textbf{THEN} there are two cases \\
$\; \; \; \;$ \textit{Case 1:} $r_{\pi(i)} > s_i$. In this case there are further three subcases: \\
$\; \; \; \; \; \;$ \textit{Subcase 1:} $r_{\pi(i)} = r_{\pi(j)}$. In this subcase, \textbf{IF} $q_{\pi(j)} > q_{\pi(i)}$ \textbf{THEN} set $\pi := \pi \circ (i, \, j)$ (i.e. swap jobs in position $i$ and $j$ within the schedule $\pi$) \textbf{ELSE} set $\pi := \pi$ (i.e. do nothing).\\
\textit{Subcase 2:} $r_{\pi(i)} > r_{\pi(j)}$ In this subcase set $\pi := \pi \circ (i, \, j)$ (i.e. perform the swap of jobs in positions $i$ and $j$ within the schedule $\pi$).\\
\textit{Subcase 3:} $r_{\pi(i)} < r_{\pi(j)}$. In this subcase set $\pi := \pi$ (i.e. do nothing)\\
\textit{Case 2:} $r_{\pi(i)} \leq s_i$. In this case \textbf{IF} $r_{\pi(j)} \leq s_i$ and $q_{\pi(j)} > q_{\pi(i)}$ \textbf{THEN} set $\pi := \pi \circ (i, \, j)$ \textbf{ELSE} set $\pi := \pi$ (i.e. do nothing).\\
\textbf{ELSE} set $\pi := \pi$ (i.e. do nothing).\\
$\;$

Essentially, the local mutation is designed to modify the schedule to become $(\delta, \, \phi)$-Jackson step by step as described precisely below.
\begin{lem}\label{gradualModificationProp}
Let $\pi$ be a generic schedule (see definition~\ref{genericScheduleDefn} and the comments that follow) and let $k$ be the maximal integer such that $\pi$ is $(k, \, \delta, \, \phi)$-Jackson. (Notice that such $k$ always exists since every schedule is $(0, \, \delta, \, \phi)$-Jackson.) After an application of the local mutation operator $\mu^{\text{local}}_{\delta}$ if $k < n$ then with probability bigger than $\frac{1}{n^2}$ the schedule $\pi$ is updated to become $(k+1, \, \delta, \, \phi)$-Jackson and otherwise $k$ remains the maximal integer such that $\pi$ is $(k, \, \delta, \, \phi)$-Jackson. If, on the other hand, the schedule $\pi$ is already $(\delta, \, \phi)$-Jackson (i.e. $k=n$) then it remains such with probability $1$.
\end{lem}
\begin{IEEEproof}
By proposition~\ref{distinctDeliveryTimesProp} the jobs in positions $1$ through $k$ in the schedule $\pi$ are uniquely determined and there exists a unique job $J_l$ to be swapped with the job in position $J_{\pi(k+1)}$ for the schedule to become $(k+1, \, \delta, \, \phi)$-Jackson unless $k=n$. Once again, by proposition~\ref{distinctDeliveryTimesProp}, the job $J_l$ must appear after the job $J_{\pi(k+1)}$ within the schedule $\pi$ i.e. $l=\pi(j)$ a unique $j > k+1$ and, therefore, in accordance with the way local mutation operator $\mu^{\text{local}}_{\delta}$ has been introduced above together with definition~\ref{PartlyJacksonDynamic}, it follows that the unique pair of positions to be selected for the successful swap is precisely $k+1$ and $j$ and this happens with probability $\frac{1}{{n \choose 2}/2} > \frac{1}{n^2}$. At the same time, once again from the construction of the local mutation operator $\mu^{\text{local}}_{\delta}$ and definition~\ref{PartlyJacksonDynamic}, it follows easily that the only transpositions (swaps) that can take place involve positions beyond $\pi(k)$ so that if the schedule does not become $(k+1, \, \delta, \, \phi)$-Jackson it remains $(k, \, \delta, \, \phi)$-Jackson at least.
\end{IEEEproof}
The following lemma gives us an important and simple polynomial time bound on the runtime complexity of the local mutation operator to reach a $(\delta, \, \phi)$-Jackson schedule provided no other mutation takes place.
\begin{lem}\label{localMutationComplexityLem}
Suppose a given initial sequence $\{J_i\}_{i=1}^n$ of jobs is generic (see definition~\ref{genericScheduleDefn}). Choose a subset $I \subseteq \{1, \, 2, \ldots, n\}$ of indices with $|I| = |B_{\delta}|$ and a bijection $\phi: I \rightarrow B_{\delta}$. Consider the Markov chain $\mathcal{M}$ on the state space $\mathcal{X}(\phi)$ (recall remark~\ref{partitionIntoPhiJacksonRem} and definition~\ref{JacksonPartitionDefn}) determined by the Markov transition matrix $\{p_{\pi \rightarrow \sigma}\}_{\pi \text{ and } \sigma \in \mathcal{X}(\phi)}$ where $p_{\pi \rightarrow \sigma}$ is the probability that the permutation (schedule) $\sigma$ is obtained from the schedule $\pi$ via an application of the local mutation operator as described above. Then $\mathcal{M}$ is an absorbing Markov chain with the unique absorbing state $\widetilde{\pi}(\phi)$ being the $(\delta, \, \phi)$-Jackson schedule. For a schedule $\pi \in \mathcal{X}(\phi)$ let $T_{\pi}$ denote the random variable measuring the time to absorbtion when the chain starts at the state $\pi$. Then $\forall \, \pi \in \mathcal{X}(\phi)$ we have $Pr(T_{\pi} > n^4) < ne^{-\Omega(n)}$.\footnote{As a matter of fact, thanks to theorem~\ref{propOfe}, namely, the property in \ref{exponentiationOfeEq}, the constant in the big $\Omega$ notation in the exponent can be made arbitrarily close to $1$ as long as it is smaller than 1.}
\end{lem}
\begin{IEEEproof}
Notice that $\mathcal{X}(\phi) = \bigcup_{i=0}^n \mathcal{X}(\phi)_k$ where $$\mathcal{X}(\phi)_k = \{\pi \, | \, \pi \text{ is }(k, \, \delta, \, \phi) \text{-Jackson}\}.$$ From the construction of the local mutation operator and proposition~\ref{distinctDeliveryTimesProp} it is clear that for $\pi \in \mathcal{X}(\phi)_k$ $\exists !$ $\sigma \in \mathcal{X}(\phi)_{k+1}$ such that $p_{\pi \rightarrow \sigma} = \frac{1}{n^2}$ while $p_{\pi \rightarrow \tau} = 0$ whenever $\tau \neq \sigma$. Thereby, the waiting time random variable
\begin{equation}\label{watingToAbserbtionSumDecompose}
T_{\pi} = \sum_{i=k}^{n-1} T_{\pi_{k} \rightarrow \pi_{k+1}}
\end{equation}
where $\pi = \pi_k, \, \pi_{k+1}, \ldots, \pi_{n-1}, \pi_n = \widetilde{\pi}(\phi)$ is the unique path to reach the $(\delta, \, \phi)$-Jackson schedule $\widetilde{\pi}(\phi)$ starting with the schedule $\pi_k$ via consecutive local mutations while $T_{\pi_{k} \rightarrow \pi_{k+1}}$ denotes the waiting time between the corresponding successive states. Notice that the random variables $T_{\pi_{k} \rightarrow \pi_{k+1}}$ are geometric and each is bounded above by a geometrically distributed random variable $G$ having success probability $\frac{1}{n^2}$. Thus $\forall \, k$ with $0 \leq k \leq n-1$ the probability
\begin{equation}\label{boundOnLocalMutTime1}
Pr(T_{\pi_{k} \rightarrow \pi_{k+1}} > n^3) \leq \left(1 - \frac{1}{n^2} \right)^{n^3} = e^{-\Omega(n)}
\end{equation}
Since there are totally at most $n$ summands in decomposition~\ref{watingToAbserbtionSumDecompose} we obtain $$Pr(T_{\pi} > n^4) \leq \sum_{k=0}^{n-1}Pr(T_{\pi_{k} \rightarrow \pi_{k+1}} > n^3) \overset{\text{by \ref{boundOnLocalMutTime1}}}{\leq}n \cdot e^{-\Omega(n)}$$ which is precisely the desired conclusion.
\end{IEEEproof}

We now proceed to define the global mutation operator in detail:

$\;$

\textbf{Global Mutation Operator $\mu^{\text{global}}$:} Given a schedule (i.e. a permutation) $\pi$ on $(1, \, 2, \ldots, n)$ select a permutation $\sigma$ on $\{1, \, 2, \ldots, n\}$ uniformly at random and set $\pi:= \sigma$.

$\;$

In words, the global mutation operator simply reorders all the jobs within a given schedule uniformly at random. The property of global mutation that is behind the design of our hybrid $2+2$-EA is the following:
\begin{lem}\label{globalMutPropLem}
Suppose a given initial sequence $\{J_i\}_{i=1}^n$ of jobs is generic (see definition~\ref{genericScheduleDefn}). Let $\epsilon > 0$ be given so that $\delta = P \cdot \epsilon$. Consider the following simple algorithm:

$\,$

Start with a permutation (schedule) $\pi$ on $\{1, \, 2, \ldots, n\}$.\\
\textbf{Repeat}\\
$\{$ With probability $\frac{1}{n^5}$ apply a global mutation operator to the schedule $\pi$ as above. Otherwise, apply a local mutation operator (a local mutation operator is applied with probability $1 - \frac{1}{n^5}$). $\}$

$\,$

Let $\vec{\pi} = (\pi_0, \, \pi_1, \ldots, \pi_l, \ldots)$ denote a typical sequence of outputs of the algorithm above (i.e. an element of the probability space for the induced stochastic (in fact, Markovian) process). Let $\pi_t$ denote a schedule after $t$ repetitions of the algorithm. Consider the following events. Let $l$ and $m$ denote random times (i.e. $\mathbb{N}$-valued random variables) and let
\begin{equation}\label{noMutationSetEqDefn}
N(l, \, m) = \{\vec{\pi} \, | \, l \leq j < i \leq m \Longrightarrow \pi_j \sim \pi_i\}.
\end{equation}
In words, $N(l, \, m)$ is the event that no global mutation takes place between times $l$ and $m$ (recall remark~\ref{partitionIntoPhiJacksonRem}, definition~\ref{JacksonPartitionDefn} and lemma~\ref{localMutationComplexityLem}). Then, as long $l \leq m \leq l+n^4$ almost surely, we have $Pr(N(l, \, m)) \geq 1 - \frac{1}{\Omega(n)}$.

In addition to the random times $l$ and $m$ as above, the next family of events we consider depends on the choices of $I \subseteq \{1, \, 2, \ldots, n\}$ with $|I| = |B_{\delta}|$ and $\phi: I \rightarrow B_{\delta}$.
\begin{equation}\label{EncounterEventEq}
Y(l, \, m, \, \phi) = \{\vec{\pi} \, | \, \exists \, j \text{ s. t. } l \leq j \leq m \text{ and } \pi_j \in \mathcal{X}(\phi)\}.
\end{equation}
In words, $Y(l, \, m, \, \phi)$ is the event that a $(k, \, \delta, \, \phi)$-Jackson schedule has been encountered at least once between the times $l$ and $m$ (see remark~\ref{partitionIntoPhiJacksonRem} and definition~\ref{JacksonPartitionDefn}). We claim that whenever $m \geq l+n^{\frac{1}{\epsilon} + 6}$ we have $Pr(Y(l, \, m, \, \phi)) \geq 1 - e^{\Omega(n)}$.\footnote{Just as in lemma~\ref{localMutationComplexityLem} with the exception that the dual of the property expressed in \ref{exponentiationOfeEq}, namely the one expressed in \ref{exponentiationOfeDualEq} is used instead, the constant in the big $\Omega$ notation in both inequalities can be made arbitrarily close to $1$ as long as it is smaller than 1.}
\end{lem}
\begin{IEEEproof}
Notice that $Pr(N(l, \, m))$ is the probability that no global mutation takes place during the consecutive $m-l$ steps. As long as these consecutive ``no occurrence" events happen independently with probabilities $1 - \frac{1}{n^5}$ during at most $n^4$ consecutive time steps it follows that $$Pr(N(l, \, m)) \geq \left(1 - \frac{1}{n^5}\right)^{n^4} = \left(\left(1 - \frac{1}{n^5}\right)^{n^5}\right)^{\frac{1}{n}}  \geq e^{-\frac{1}{\Omega(n)}} \geq 1 - \frac{1}{\Omega(n)}$$ as claimed.

Since there are totally less than $n^{\frac{1}{\epsilon}}$ choices of the possible bijections $\phi$ and the global mutations take place independently with probability $\frac{1}{n^5}$ selecting each such $\phi$ uniformly at random, it follows that the probability of obtaining a $(k, \, \delta, \, \phi)$-Jackson schedule after any single time step is at least $\frac{1}{n^{\frac{1}{\epsilon}+5}}$. Therefore, the probability of not encountering a $(k, \, \delta, \, \phi)$-Jackson schedule after any single time step is at most $1 - \frac{1}{n^{\frac{1}{\epsilon}+5}}$. Once again, by independence together with the assumption that $m \geq l+n^{\frac{1}{\epsilon} + 6}$ we deduce that $$Pr(\overline{Y(l, \, m, \, \phi)}) \leq \left(\left(1 - \frac{1}{n^{\frac{1}{\epsilon}+5}}\right)^{n^{\frac{1}{\epsilon} + 5}}\right)^n = e^{-\Omega(n)}$$ where $\overline{Y(l, \, m, \, \phi)}$ denotes the event that no $(k, \, \delta, \, \phi)$-Jackson schedule is ever encountered between $l^{\text{th}}$ and $m^{\text{th}}$ time steps; in other words, the complement of the event $Y(l, \, m, \, \phi)$. The desired conclusion that $$Pr(Y(l, \, m, \, \phi)) = 1 - Pr(\overline{Y(l, \, m, \, \phi)}) \geq 1 - e^{-\Omega(n)}$$ now follows and finishes the proof.
\end{IEEEproof}
The algorithm in the statement of lemma~\ref{globalMutPropLem} motivates the construction of a hybrid $2+2$ EA to minimize the maximum lateness of a schedule up to an approximation ratio of $(1+\epsilon)$ within time polynomial in $n$ (the number of jobs in the schedule) and $\frac{1}{\epsilon}$ with an overwhelmingly high probability. Essentially, the first individual in the population undergoes consecutive mutations exactly as in the algorithm described in the statement of lemma~\ref{globalMutPropLem} while the second individual stores the fittest schedule (i.e. the schedule having the minimal maximal lateness as in definition~\ref{startingTimeDef}) encountered up to the current time.

$\;$

\textbf{Hybrid $2+2$-EA:} The initial population $P_0 = (\pi(0), \, \sigma(0))$ where $\pi(0)$ and $\sigma(0)$ are arbitrary schedules (i.e. permutations on $\{1, \, 2, \ldots, n\}$). At every iteration of the algorithm the following takes place: given a population $P_t = (\pi(t), \, \sigma(t))$ after $t$ time steps of the algorithm, let $\pi(t+1)$ be obtained from $\pi(t)$ via an application of the global mutation operator with probability $\frac{1}{n^5}$ and via an application of the local mutation operator with a significantly higher complementary probability of $1 - \frac{1}{n^5}$. \textbf{IF} $L_{\pi(t+1)} < L_{\sigma(t)}$ \textbf{THEN} set $\sigma(t+1):=\pi(t+1)$ \textbf{ELSE} set $\sigma(t+1):=\sigma(t)$. Now obtain the new generation after $t+1$ time steps, $P_{t+1} = (\pi(t+1), \, \sigma(t+1))$.

$\;$

\begin{thm}\label{finalGoalThm}
Suppose we start with a generic (see definition~\ref{genericScheduleDefn} and discussion that follows) sequence of jobs $\{J_i\}_{i=1}^n$ and the Hybrid $2+2$-EA described above runs for at least $n^{\frac{1}{\epsilon} + 6}+n^4$ time steps. Then the probability that a schedule having the maximal lateness smaller than $(1+\epsilon)J^*$ has not been encountered is at most $\frac{1}{\Omega(n)}$ where the constant in the $\Omega$ notation can be made arbitrarily close to $1$. More generally, the probability that a schedule having the maximal lateness smaller than $(1+\epsilon)J^*$ has not been encountered after $\lambda\left(n^{\frac{1}{\epsilon} + 6}+n^4\right)$ time steps is at most $\Omega(n)^{-\lambda}$. In particular, the probability that such a schedule is not found after running the hybrid $2+2$-EA at least $n^{\frac{1}{\epsilon} + 7}+n^5$ time steps is $\Omega(n)^{-n}$ which is overwhelmingly small.
\end{thm}
\begin{IEEEproof}
Recall from theorem~\ref{mainThmForDesign} that one can always choose a suitable subset of indices $I \subseteq \{1, \, 2, \ldots n\}$ and a bijection
$\phi^*: I \rightarrow B_{\delta}$ such that any $(\delta, \, \phi^*)$-Jackson schedule $\pi^*$ achieves the maximal lateness no bigger than $(1+\epsilon)J^*$. Under the mild assumption that the given sequence of jobs $\{J_i\}_{i=1}^n$ is generic (see proposition~\ref{distinctDeliveryTimesProp}, definition~\ref{genericScheduleDefn} and the discussion following definition~\ref{genericScheduleDefn}) such a $(\delta, \, \phi^*)$-Jackson schedule is unique. According to lemma~\ref{globalMutPropLem} a $(k, \, \delta, \, \phi^*)$-Jackson schedule will be encountered at least once as the first individual of the population between times $0$ and $n^{\frac{1}{\epsilon} + 6}$ is at least $1 - e^{-\Omega(n)}$. In other words, the probability of the event
\begin{equation}\label{probOfEncounterBestEquivEq}
Pr\left(Y(0, \, n^{\frac{1}{\epsilon} + 6}\, \phi^*)\right) \geq 1 - e^{-\Omega(n)}.
\end{equation}
Now let $l_{\text{first}}$ denote the first (random) time when a $(k, \, \delta, \, \phi^*)$-Jackson schedule has been encountered and consider the event $N(l_{\text{first}}, \, l_{\text{first}} + n^4)$ as in lemma~\ref{globalMutPropLem}. Recall from lemma~\ref{globalMutPropLem} that $N(l_{\text{first}}, \, l_{\text{first}} + n^4)$ is the event that no global mutation takes place between the times $l_{\text{first}}$ and $l_{\text{first}} + n^4$ and the fact that
\begin{equation}\label{noGlobalMutationBoundEq}
Pr\left(N(l_{\text{first}}, \, l_{\text{first}} + n^4)\right) \geq 1 - \frac{1}{\Omega(n)}
\end{equation}
Since our goal is to encounter the $(\delta, \, \phi^*)$-Jackson schedule, we also need to consider the random time $r_{\text{first}}$ when the $(\delta, \, \phi^*)$-Jackson schedule has been encountered for the first time after the first $(k, \, \delta, \, \phi^*)$-Jackson schedule has been encountered as well as the events \begin{equation}\label{localGoalReachedEqDefn}
E = \{\vec{\pi} \, | \, r_{\text{first}} \leq l_{\text{first}} + n^4\} \text{ and } E \cap N(l_{\text{first}}, \, l_{\text{first}} + n^4).
\end{equation}
Given that the event $N(l_{\text{first}}, \, l_{\text{first}} + n^4)$ has taken place, according to lemma~\ref{localMutationComplexityLem}, the conditional probability
\begin{equation}\label{conditionaProbIneqal}
Pr\left(E \, | \, N(l_{\text{first}}, \, l_{\text{first}} + n^4)\right) \geq 1 - n \cdot e^{-\Omega(n)}
\end{equation}
so that
\begin{align}
Pr\left(E \cap N(l_{\text{first}}, \, l_{\text{first}} + n^4)\right) &=Pr\left(E \, | \, N(l_{\text{first}}, \, l_{\text{first}} + n^4)\right) \cdot Pr\left(N(l_{\text{first}}, \, l_{\text{first}} + n^4)\right) \nonumber \\&\overset{\text{thanks to inequalities~\ref{conditionaProbIneqal} and \ref{noGlobalMutationBoundEq}}}{\geq} \left(1 - n \cdot e^{-\Omega(n)}\right) \cdot \left(1 - \frac{1}{\Omega(n)} \right) \nonumber\\
& \geq 1 - \frac{1}{\Omega(n)} \label{inersectProbIntermEq}
\end{align}
In words, the event
\begin{equation}\label{localGoalEventDefnEq}
\text{LocalGoal}_{\text{first}} = E \cap N(l_{\text{first}}, \, l_{\text{first}} + n^4)
\end{equation}
is the event that no global mutation takes place between the times $l_{\text{first}}$ and $l_{\text{first}} + n^4$ and, at the same time, the $(\delta, \, \phi^*)$-Jackson schedule has been reached between the times $l_{\text{first}}$ and $l_{\text{first}} + n^4$ (using local mutation operators only). We summarize inequality~\ref{inersectProbIntermEq} as
\begin{equation}\label{intersectSummaryInequal}
Pr \left(\text{LocalGoal}_{\text{first}}\right) \geq 1 - \frac{1}{\Omega(n)}.
\end{equation}
Notice the important fact that the events $\text{LocalGoal}_{\text{first}}$ and $Y(0, \, n^{\frac{1}{\epsilon} + 6}\, \phi^*)$ are independent: indeed, according to the strong Markov property, $$Pr(\text{LocalGoal}_{\text{first}} |  Y(0, \, n^{\frac{1}{\epsilon} + 6}\, \phi^*))  =Pr(\text{LocalGoal}_{\text{first}} \, | \, \overline{Y(0, \, n^{\frac{1}{\epsilon} + 6}\, \phi^*)} \, )$$ where $\overline{Y(0, \, n^{\frac{1}{\epsilon} + 6}\, \phi^*)}$ denotes the complement of the event $Y(0, \, n^{\frac{1}{\epsilon} + 6}\, \phi^*)$, so that the desired conclusion follows from an elementary property of conditional probability.

Consider now the sub-event
\begin{equation}\label{intermEventEqDefn}
Z(0, \, n^{\frac{1}{\epsilon} + 6}\, \phi^*) = Y(0, \, n^{\frac{1}{\epsilon} + 6}\, \phi^*) \cap \text{LocalGoal}_{\text{first}}.
\end{equation}
In words, $Z(0, \, n^{\frac{1}{\epsilon} + 6}\, \phi^*)$ is the event that a $(k, \, \delta, \, \phi^*)$-Jackson schedule has been encountered some time within the first $n^{\frac{1}{\epsilon} + 6}$ time steps and, as soon as it was encountered, no global mutation took place during consecutive $n^4$ time steps and, during these consecutive $n^4$ time steps, the $(\delta, \, \phi^*)$-Jackson schedule has been encountered. In particular, if the event $Z(0, \, n^{\frac{1}{\epsilon} + 6}\, \phi^*)$ takes place then the $(\delta, \, \phi^*)$-Jackson schedule has been encountered within at most $n^{\frac{1}{\epsilon} + 6}+n^4$ time steps so that the event
\begin{equation}\label{targetEventEqDefn}\text{Target}(n^{\frac{1}{\epsilon} + 6}+n^4) = \{\vec{\pi} \, | \, \exists \, j \text{ with } 0 \leq j \leq n^{\frac{1}{\epsilon} + 6}+n^4 \mbox{ and } \pi_j \text{ is } (\delta, \, \phi^*)-\text{Jackson}\} \supseteq Z(0, \, n^{\frac{1}{\epsilon} + 6}\, \phi^*)
\end{equation}
Thanks to the set inclusion relation expressed in \ref{targetEventEqDefn} together with the independence of the events $\text{LocalGoal}_{\text{first}}$ and $Y(0, \, n^{\frac{1}{\epsilon} + 6}\, \phi^*)$, we finally deduce that
\begin{align}Pr(\text{Target}(n^{\frac{1}{\epsilon} + 6}+n^4)) &\geq Pr\left(Z(0, \, n^{\frac{1}{\epsilon} + 6}\, \phi^*)\right)  \nonumber \\
 &=Pr\left( Y(0, \, n^{\frac{1}{\epsilon} + 6}\, \phi^*) \right) \cdot Pr\left( \text{LocalGoal}_{\text{first}} \right)  \nonumber \\
 &\overset{\text{from inequalities \ref{probOfEncounterBestEquivEq} and \ref{intersectSummaryInequal}}}{\geq} \left(1 - e^{-\Omega(n)}\right) \cdot \left( 1- \frac{1}{\Omega(n)}\right)\nonumber \\
 \label{finalDeductionIneqal}
&= 1- \frac{1}{\Omega(n)}
\end{align} 
Inequality~\ref{finalDeductionIneqal} tells us that the $(\delta, \, \phi^*)$-Jackson schedule is encountered as the first individual in the population of our hybrid $2+2$-EA within the first $n^{\frac{1}{\epsilon} + 6}+n^4$ time steps with probability at least $1- \frac{1}{\Omega(n)}$. According to theorem~\ref{mainThmForDesign} the $(\delta, \, \phi^*)$-Jackson schedule provides the desired approximation of the optimal fitness with the approximation ratio of $(1 + \epsilon)$ and the first assertion of theorem~\ref{finalGoalThm} is established. To see the next assertion, consider the complementary event $\mathcal{N}(1)$ that the optimal schedule is not encountered during the first $n^{\frac{1}{\epsilon} + 6}+n^4$ time steps of our hybrid $2+2$-EA. Clearly
\begin{equation}\label{negativeEventBoundIneq}
Pr \left(\mathcal{N}(1)\right) \leq 1 - \left(1- \frac{1}{\Omega(n)} \right) = \frac{1}{\Omega(n)}
\end{equation}
By the Markov property the probability of the event $\mathcal{N}(1 \leftrightarrows 2)$ that the $(\delta, \, \phi^*)$-Jackson schedule is never encountered after the
first $n^{\frac{1}{\epsilon} + 6}+n^4$ time steps and before $2 \left(n^{\frac{1}{\epsilon} + 6}+n^4\right)$ time steps is bounded above in the same way as the probability of the event $\mathcal{N}(1)$ and therefore it follows that the probability that
the $(\delta, \, \phi^*)$-Jackson schedule is never encountered before the first consecutive $2 \left(n^{\frac{1}{\epsilon} + 6}+n^4\right)$ time steps is
at most $Pr(\mathcal{N}(1)) \cdot Pr(\mathcal{N}(1 \leftrightarrows 2)) \leq \frac{1}{\Omega(n)} \cdot \frac{1}{\Omega(n)} = \left(\frac{1}{\Omega(n)}\right)^2$.
Continuing in this manner inductively, we deduce that $\forall \, \lambda \in \mathbb{N}$ the probability that the $(\delta, \, \phi^*)$-Jackson schedule is never encountered after the first $\lambda \left(n^{\frac{1}{\epsilon} + 6}+n^4\right)$ time steps is at most $\Omega(n)^{-\lambda}$. In particular, when $\lambda = n$ we deduce that the probability of never encountering the $(\delta, \, \phi^*)$-Jackson schedule after $n^{\frac{1}{\epsilon} + 7}+n^5$ time steps is at most $\Omega(n)^{-n}$ immediately implying the last conclusion.
\end{IEEEproof}
\section{Conclusions and Future Work}
In the current paper we have designed a parameterized family of hybrid $2+2$ EAs for the single machine scheduling problem as described in section~\ref{technicalPartSect} (alternatively, see \cite{LeslieHall} for a more detailed description of the single machine scheduling problem with no precedence constraints). The hybrid EA combines local and global mutation operators. The local mutation drives towards a ``partial Jackson rule" where the notion of ``partial" depends largely on the approximation parameter $\epsilon$, while the global mutation operator is the pure random search. The main result of this work is theorem~\ref{finalGoalThm} where we have shown that for a given approximation parameter $\epsilon$ and the corresponding local mutation operator that depends on $\delta = \epsilon \cdot P$ where $P = \sum_{i=1}^n p_i$ is the sum of all the processing times of the given jobs to schedule, after $\lambda \left(n^{\frac{1}{\epsilon} + 6}+n^4\right)$ time steps our hybrid EA achieves a solution having fitness $(1+\epsilon) J^*$ with probability $\Omega(n)^{-\lambda}$ where $J^*$ denotes the optimum fitness and the constant in the $\Omega$ notation can be made arbitrarily close to $1$ (see definition~\ref{asymptoticNotationDefn} and inequalities~\ref{exponentiationOfeEq} and \ref{exponentiationOfeDualEq} within theorem~\ref{propOfe})  Our design and complexity analysis is based on rather elementary tools presented in section~\ref{toolsSect} together with basic Markov chain theory. Since hybrid evolutionary algorithms are rather popular nowadays and many of them combine local mutations driving towards an implementation of a certain well-known algorithm for a specific optimization problem, and global mutations which serve as some kind of a random search (see \cite{goh2009hybrid}, \cite{burke2003hyper}, \cite{Hart2005}, \cite{xhafa2007hybrid}, \cite{bambha2004systematic}, \cite{galinier1999hybrid} and \cite{preux1999towards} as well as a rather extensive discussion in the introduction) we hope that the mathematical ideas and techniques presented in this article can be successfully applied to analyze many other hybrid evolutionary algorithms.



%
\subsubsection*{Acknowledgement:}  This work is  supported by  the EPSRC under Grant EP/I009809/1.

\def\V{\rm vol.~}
\def\N{no.~}
\def\pp{pp.~}
\def\Pot{\it Proc. }
\def\IJCNN{\it International Joint Conference on Neural Networks\rm }
\def\ACC{\it American Control Conference\rm }
\def\SMC{\it IEEE Trans. Systems\rm , \it Man\rm , and \it Cybernetics\rm }

\def\handb{ \it Handbook of Intelligent Control: Neural\rm , \it
    Fuzzy\rm , \it and Adaptive Approaches \rm }

\end{document}